\documentclass[conference]{IEEEtran}
\usepackage{amsthm,amssymb,graphicx,multirow,amsmath,color,amsfonts}
\usepackage[update,prepend]{epstopdf}
\usepackage[noadjust]{cite}
\usepackage[latin1]{inputenc}
\usepackage{tikz}
\usepackage{bbm} 
\usepackage{pdfpages}
\usepackage{multirow}
\usepackage{comment}
\def\chr#1{{\color{red}#1}}
\def\chb#1{{\color{blue} #1}}
\usepackage{algorithm}
\usepackage[noend]{algpseudocode}

\usepackage{balance}

\newcommand{\cO}{{\cal O}}

\newtheorem{corollary}{Corollary}[section]

\newtheorem{lemma}{Lemma}[section]

\newtheorem{proposition}{Proposition}[section]

\newtheorem{assumption}{Assumption}[section]

\usepackage{subfig}
\include{notation}
\allowdisplaybreaks 
\usepackage[colorinlistoftodos,bordercolor=orange,backgroundcolor=orange!20,linecolor=orange,textsize=scriptsize]{todonotes}
\newcommand{\houcine}[1]{\todo[inline]{\textbf{Houcine: }#1}} 
\newcommand{\hajar}[1]{\todo[inline]{\textbf{Hajar: }#1}} 
\newcommand{\ouiame}[1]{\todo[inline]{\textbf{Ouiame: }#1}} 
\usepackage{setspace}	

\graphicspath{{figures/}}

\newcommand{\algorithmicserver}{\textbf{Global Server}}
\algdef{SE}[SERVER]{Server}{EndServer}[1]{\algorithmicserver\ #1\ \algorithmicdo}{\algorithmicend\ \algorithmicserver}%
\algtext*{EndServer}
\newcommand{\algorithmicclient}{\textbf{Client k}}
\algdef{SE}[CLIENT]{Client}{EndClient}[1]{\algorithmicclient\ #1\ \algorithmicdo}{\algorithmicend\ \algorithmicclient}%
\algtext*{EndClient}
\makeatletter
\newcommand{\brokenline}[2][t]{\parbox[#1]{\dimexpr\linewidth-\ALG@thistlm}{\strut\raggedright #2\strut}}
\makeatother

\usepackage{romanbar}
\newcommand\norm[1]{\left\lVert#1\right\rVert}
\makeatletter
\renewcommand{\fnum@figure}{Figure \thefigure}
\makeatother

\makeatletter
\renewcommand{\fnum@table}{Table \thetable}
\makeatother

\begin{document}
\title{
Client Selection in Federated Learning based on Gradients Importance
}

\author{
    \IEEEauthorblockN{Ouiame Marnissi, Hajar EL Hammouti, El Houcine Bergou
        }
    \IEEEauthorblockA{ School of Computer Science, Mohammed VI Polytechnic University, Ben Guerir, Morocco.\\
\{ouiame.marnissi, hajar.elhammouti, elhoucine.bergou\}@um6p.ma
}
}
\maketitle

\begin{abstract}












Federated learning (FL) enables multiple devices to collaboratively learn a global model without sharing their personal data. In real-world applications, the different parties are likely to have heterogeneous data distribution and limited communication bandwidth. 
In this paper, we are interested in improving the communication efficiency of FL systems. We investigate and design a device selection strategy based on the importance of the gradient norms.  In particular, our approach consists of selecting devices with the highest norms of gradient values at each communication round. We study the convergence and the performance of such a selection technique and compare it to existing ones. We perform several experiments with non-iid set-up. The results show the convergence of our method with a considerable increase of test accuracy comparing to the random selection. 

\end{abstract}

\begin{IEEEkeywords}
Federated learning, Gradient descent, Non-iid data
\end{IEEEkeywords}
\section{Introduction} \label{sec:intro}
Machine learning (ML) has emerged as a promising technique that captures the data patterns and performs accurate predictions~\cite{lecun1998gradient}. In a classic ML setup, data is collected from different sources. It is uploaded to a centralized server, processed, and then used to train ML algorithms. However, this centralized framework results in two major issues. First, sharing data with a centralized entity may compromise the user's privacy. Second, due to the huge volumes of shared information, the task of uploading raw data through the network is prohibitively expensive, and sometimes unpractical.

In this context, FL has been developed to preserve the user's privacy and reduce the amount of transmitted information~\cite{McMahanMRA16,wahab2021federated}. Indeed, instead of sharing raw data, ML models are trained locally on devices, and only their parameters are sent to a centralized server. As a consequence, the data of the users is preserved, and less information is transmitted through the network. One of the most popular FL algorithms is \textit{federated averaging} \cite{McMahanMRA16}. In order to train a global ML model, federated averaging let a subset of users perform local trainings. After a number of local iterations, the local gradients are sent to the central server for aggregation. These steps are repeated until a stable global model is obtained or a target accuracy is achieved. Figure \ref{flprocess} describes this process.
\begin{figure*}
    \centering
    \includegraphics[scale=0.45]{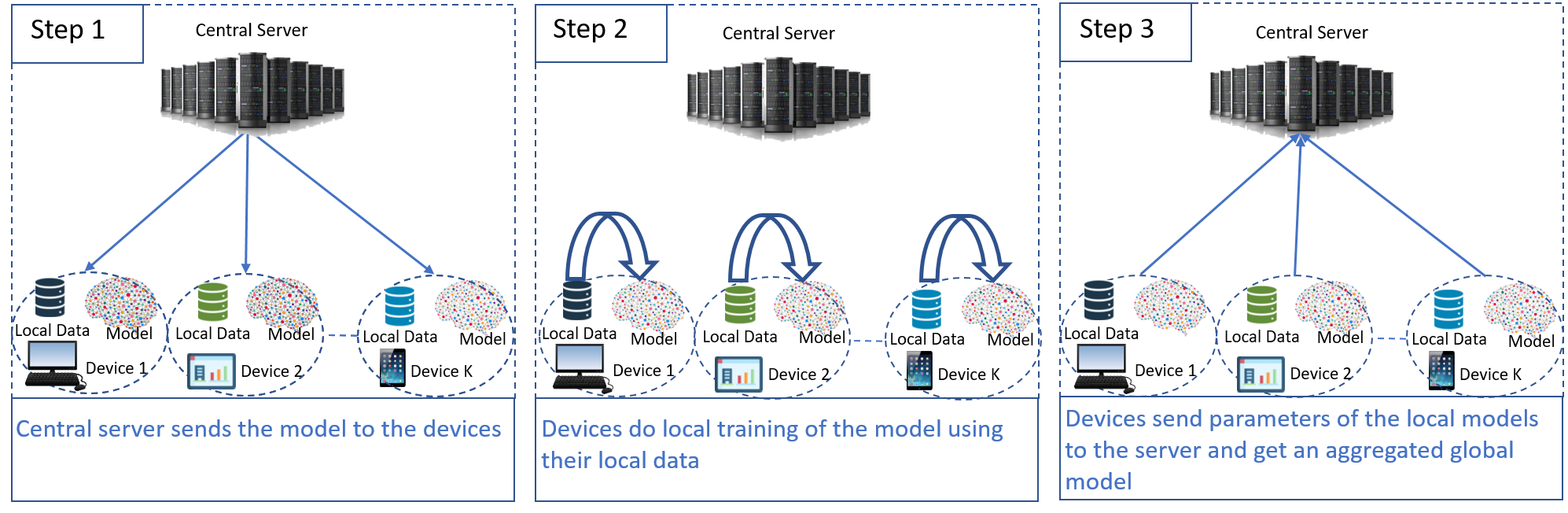}
    \caption{Federated Learning Process}
    \label{flprocess}
\end{figure*}

However, current ML applications (e.g., image and voice recognition, word prediction) involve high-dimensional ML models. As a consequence, large gradient vectors are sent over resource-constrained networks~\cite{konevcny2016federated}. To reduce the communication overhead, compression and quantization techniques have been proposed~\cite{aji2017sparse, Grace}. The main purpose of these techniques is to reduce the size of transmitted vectors while maintaining  good performance of FL algorithms. Another approach to preserve the scarce network resources is to efficiently select the subset of participating devices. Not only partial client participation uses limited communication bandwidth, but when optimally designed, it can also accelerate the FL convergence and minimize the computational resources. However, while convergence for full client participation with arbitrarily heterogeneous data is guaranteed, the convergence of partial device participation is challenging and depends heavily on the selection approach. 

For example, in~\cite{nishio2019client}, the authors propose a greedy algorithm to maximize the number of selected clients based on their computational and communication capabilities. In particular, at each communication round, devices with the minimum transmission and computation times are selected. While the proposed technique provides promising experimental results, it is not supported by any convergence analysis. A more realistic scenario is studied in~\cite{ruan2021towards}. In fact, the authors study the case where devices are unable to complete the learning task due to energy depletion or connectivity disruption. They propose a federated averaging scheme where the aggregation is weighted by probabilities of devices being inactive at a given communication round. Another selection approach is proposed in~\cite{ribero2020communication}. The approach suggests that clients with the most significant local updates are selected. The scheme is combined with Ornstein-Uhlenbeck process to estimate the updates of clients that do not communicate their updates to the server. 
The closest work to ours is the one described in~\cite{ref3}. It is shown there that a high accuracy can be achieved when the selection of devices is biased towards clients with the highest loss values. Unlike~\cite{ref3}, our approach is based on gradients comparison. In particular, clients with the most impactful gradient norm values are selected. The main virtue of such an approach is that it can accelerate the convergence time with a reduced computation complexity.

In this paper, we address the problem of device selection in a resource-constrained network. In particular, we answer the question: how to select a limited number of devices in order to accelerate the convergence of the FL algorithm? Indeed, we propose an efficient selection technique whereby the subset of participating devices is determined based on the norm of their gradients. The remainder of this paper is organized as follows. We describe the system model and state the learning problem in section II. In section III, we introduce our selection method. We also investigate its convergence and provide insights about the convergence time. Finally, in section IV, we provide extensive simulation results to show the performance of our proposed approach. Our selection scheme is compared with two other selection techniques: highest loss and random selection approaches.


\section{System Model}
\begin{figure}
    \centering
    \includegraphics[scale=0.3]{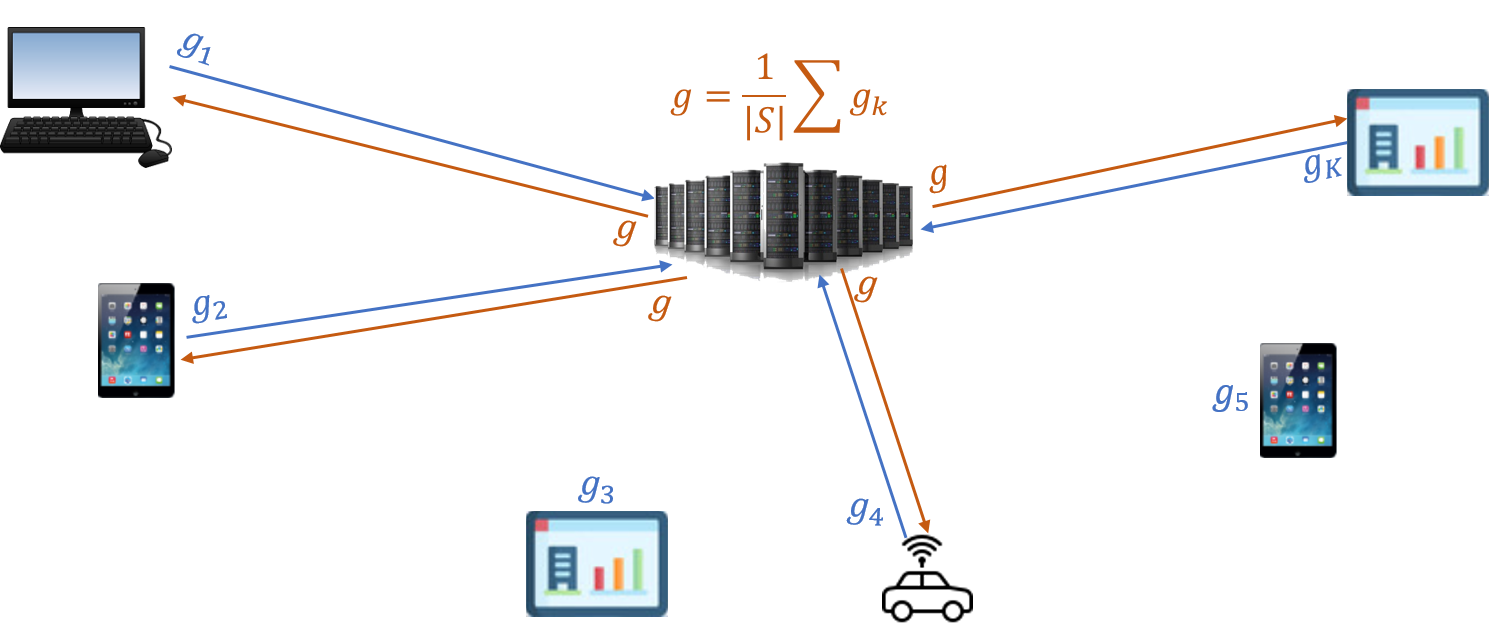}
    \caption{System Model}
    \label{sysmodel}
\end{figure}
We consider a number $K$ of connected devices that are able to communicate with a central server as in Figure \ref{sysmodel}. Each device $k$ has a local dataset $D_k$ with size $|D_k|$. We denote by $D=\cup_{k=1}^{K}D_k$ the total dataset over all devices. In order to train the ML model, a global objective function $f$ is minimized. Let $ w\in \mathbb{R}^d$ be the ML parameters vector, where $d$ is the dimension of the global model. In a typical FL, $f$ is given by the average loss over data samples at all devices $k$. Therefore, $f$ can be written as 
\begin{equation*}
f(w)=\frac{1}{\sum\limits_{k=1}^K|D_k|}\sum \limits_{k=1}^{K}f_k(w),
\end{equation*} 
where $f_k$ is the local loss function of device $k$. 

The loss function is minimized using a stochastic gradient descent method. The latter method proceeds iteratively where, at each iteration $t$, devices perform local computations of the stochastic gradients based on their local data.  Then,   
 to update the parameter vector $w$, a subset $\mathcal{S}$ of clients transmit their local gradients to the server for aggregation. The global model at iteration $t$  is updated as follows

\begin{equation}\label{GD}
     w^{t+1}= w^{t}-\eta g(w^t),
\end{equation} 
where $\eta>0$ is the learning rate, 
$g(w^t) = \frac{1}{|S^t|} \sum_{i \in \mathcal{S}^t}g_i(w^t)$, $g_i(w^t)$ is  the stochastic gradient computed by client $i$ at $w^t$ and $\mathcal{S}^t$ is the subset of selected clients at iteration $t$. Our objective is to select the subset  $\mathcal{S}^t$ efficiently so that the convergence of the FL algorithm is ensured. In the following, we describe how this subset is selected.

\section{Client Selection Approach}

In this section, we propose an efficient client selection technique to accelerate the convergence of the FL. First, we describe how our approach can be deployed within the federated averaging framework. Second, we show that our approach is guaranteed to converge when a single client, the one with the highest gradient norm, is selected at each communication round. 

\subsection{Algorithm description}

Assume the standard FL algorithm. At a given communication round $t$, the server broadcasts the parameters of the global model $w^t$ to all the clients. The clients upload the updated ML model. They compute local gradients based on their local datasets. Each client computes the Euclidean norm of its local gradient. The norm values are sent to the central server which determines, based on the best norm values, the subset of clients participating in the next iteration. Once the subset of participating clients is designated by the server, the clients send their local gradients for aggregation. In Algorithm 1, we present the pseudo-code of the proposed selection approach.

The proposed approach updates the selected clients at each communication round. This allows for better training of the ML model. It also reduces the communication and computation burden at the clients. Furthermore, although the proposed approach leads to additional communication exchange between the server and the clients, the amount of additional information (the gradient norms) is almost negligible when compared to the high dimensional gradient vectors. It is important to note that the central server acts as a coordinator that ensures the selection of the clients. However, to reduce the communication time with a faraway server, the coordination task can also be carried out by any client. The main objective of the coordinator is to compare the gradient norms and determine the fraction of the users with the highest norm values. Intuitively, the users with the highest gradient norms are, most probably, the ones with the most skewed local models. The selection of these clients allows the global model to include their datasets in the training, which helps achieve better accuracy. 

To obtain the gradients, one has not necessarily to compute the loss values. Gradients can be estimated 
directly from the ML model. Recomputing the loss function after each update can be time and resource consuming. This is particularly important when participating clients are Internet of Things devices with limited communication and computation capabilities. For example, for many ML models such as polynomial and logistic regressions, the expression of the gradient is known beforehand. This expression can be used directly to obtain the gradients and update the weights without the need to compute the loss function at each iteration.


\begin{algorithm}
	\caption{Gradient Based Selection for FL} 
	\label{alg:GBSFL}
	\begin{algorithmic}[1]\\
	$K$: number of devices indexed by $k$,
	$S$: Subset of selected clients, $\eta$: learning rate \\
	Initialize $w^0$
	
		 \For {$t=0,1,\ldots,T  $ communication rounds$ $ }
			\Server{}
				\State Send the global model $w^t$ to the clients 
				\State \brokenline{%
				Select $C$ devices based on the $C$ highest gradient norm  values}
				\State \brokenline{%
				Average the gradients of the selected devices $g(w^t) = \frac{1}{|S^t|} \sum_{i \in \mathcal{S}^t}g_i(w^t)$}
                \State Update the global model as in equation (\ref{GD})
            \EndServer
		    \Client{}	
			    \State Compute the gradient $g_k(w^t)$
			 \EndClient{}
		\EndFor
	\end{algorithmic} 
\end{algorithm}

\subsection{Convergence analysis}
In the following, we present the convergence analysis when the client with the highest gradient norm is selected. In the rest, we denote by $\norm{x}$ the vector $x$'s Euclidean norm.

\begin{lemma}
With our proposed strategy, i.e., select the client with the highest norm, we choose the biggest possible step in norm to move to the next iteration, i.e, 
\begin{equation*}
\begin{aligned}
\norm{w^{t+1}-w^t} &=  \underset{\Delta w}{\text{max}}
 \norm{\Delta w} \\
& 
  \Delta w \sim \eta g_k(w^t), \;k \in \{ 1, \ldots, K\}.
\end{aligned}
\end{equation*}
\end{lemma}
\begin{proof}
We have at iteration $t$: 
$w^{t+1} = w^t - \eta g_i(w^t)$ where $i$ is the index of the client with the highest gradient norm, i.e,
\begin{equation*}
\begin{aligned}
i &=  \underset{}{\text{arg max}}
 \norm{g_k(w^t)} \\
&    k \in \{ 1, \ldots, K\}.  
\end{aligned}
\end{equation*}
Hence,  
\begin{equation*}
\norm{w^{t+1}-w^t} =  \eta \norm{g_i(w^t)} = \max_{k \in \{ 1, \ldots, K\}} \eta \norm{g_k(w^t)}.
\end{equation*}
\end{proof}

Before showing the convergence result for Algorithm \ref{alg:GBSFL}, we  state the general assumptions we make (several of which are classical ones). 
We denote by $i$ the device with the highest gradient norm at iteration $t$.   
\begin{assumption}
\label{ass:lowerbound}
f is lower bounded by $f^*$.
\end{assumption}

\begin{assumption}
\label{ass:smoothness}
There exists $L>0$ such that 
f is $L$-smooth, i.e, for all $x$ and $y$:
 \begin{align*}
f(x)-f(y)-\nabla f(y)^T(x-y)\leq \frac{L}{2}\norm{x - y}^2
\end{align*}
\end{assumption}

\begin{assumption}
\label{ass:sgdbound}
There exists $G>0$ such that the 
stochastic gradient $g_i$ is bounded by $G$. i.e, 
$\norm{g_i(w^t)} \le G$ for all $w^t.$
\end{assumption}
The following  assumption lower bounds the expected inner product of the stochastic gradient ${g}_i(w^t)$ with the gradient $\nabla f(w^t)$ with a positive quantity depending on a power of the gradient norm while allowing a small residual on the lower bound. 
\begin{assumption}\label{ass:decreas}
There exists $\mu>0$ such that 
\begin{equation}  
\textstyle \mathbb{E} \left[{g}_i(w^t)^\top \nabla f(w^t)\right] \ge \mu   \|\nabla f(w^t)\|^2 +  R_t,
\end{equation}
where  $R_t$ is a small scalar residual which may appear due to the numerical inexactness of some operators or due to other computational overheads. 
\end{assumption}
The latter assumption generalizes the unbiasedness assumption on the stochastic gradient. In fact, if we assume the unbiasedness  of the stochastic gradient then the previous assumption holds trivially with $\mu=1$ and $R_t = 0$.
A similar assumption was proposed in \cite{ stochastic_threepoint2020, layer_wise2020} for the biased stochastic gradient descent.





We now state the convergence complexity result for Algorithm \ref{alg:GBSFL}. We mainly show similar complexity bounds known for baseline SGD and its variants. 
\begin{proposition}
\label{prop:conv}
Let Assumptions \ref{ass:lowerbound}, \ref{ass:smoothness}, \ref{ass:sgdbound} and \ref{ass:decreas}  hold, then 
 \begin{align*}
\frac{1}{T+1}\sum_{t=0}^T\mathbb{E} \norm{\nabla f(w^t)}^2 \leq \frac{f(w^0) - f^*}{(T+1) \eta \mu} + \frac{R_T}{\mu}  + \frac{L}{2 \mu}\eta G^2,
\end{align*}
where $R_T = \frac{1}{T+1}\sum_{t=0}^T R_t.$
\end{proposition}
\begin{proof}

Using Assumption \ref{ass:smoothness} we get
\begin{align}
\label{eq:smoothineq}
f(w^{t+1}) \leq f(w^t) -\eta\nabla f(w^t)^T  g_i(w^t) + \frac{L}{2} \eta^2 \norm{ g_i(w^t)}^2.
\end{align}
By taking the expectation conditional to $w^t$ and using Assumption \ref{ass:decreas} we obtain
 \begin{align*}
\mathbb{E}[f(w^{t+1})|w^t] \leq f(w^t) -\eta \mu \norm{\nabla f(w^t)}^2 + \eta R_t + \frac{L}{2} \eta^2 G^2.
\end{align*}
By taking now the expectation on the last inequality and rearranging  the terms, we have 
\begin{align*}
\mathbb{E} \norm{\nabla f(w^t)}^2 \leq  \frac{\mathbb{E} f(w^t) - \mathbb{E} f(w^{t+1})}{\eta \mu} +\frac{R_t}{\mu} + \frac{L}{2 \mu}\eta G^2.
\end{align*}
By summing over $t$ from $0$ to $T$ and using the telescopic sum we get 
\begin{align*}
\sum_{t=0}^T \mathbb{E} \norm{\nabla f(w^t)}^2 &\leq  \frac{ f(w^0) - \mathbb{E} f(w^{T+1})}{\eta \mu} + \sum_{t=0}^T \frac{R_t}{\mu} \\
&+ \frac{(T+1)L}{2 \mu}\eta G^2.
\end{align*}
From assumption~\ref{ass:lowerbound}, we have $f(w^{T+1}) \ge f^*$, thus 
\begin{align*}
\sum_{t=0}^T \mathbb{E} \norm{\nabla f(w^t)}^2 \leq  \frac{ f(w^0) -  f^*}{\eta \mu} + \sum_{t=0}^T \frac{R_t}{\mu} + \frac{(T+1)L}{2 \mu}\eta G^2.
\end{align*}
To conclude, we simply divide the last inequality by $T+1$.

\end{proof}

\begin{corollary}
Let Assumptions \ref{ass:lowerbound}, \ref{ass:smoothness}, \ref{ass:sgdbound} and \ref{ass:decreas}  hold. If $\eta = \cO \left(\tfrac{1}{\sqrt{T+1}}\right)$ and $R_T = \cO \left(\tfrac{1}{\sqrt{T+1}}\right)$ then 
\begin{eqnarray*}
 \min_{t\in[0,\ldots,T]}\mathbb{E}\|\nabla f(w^t)\|^2 \leq \cO\left(\tfrac{1}{\mu \sqrt{T+1}}\right). 
\end{eqnarray*}
\end{corollary}
\begin{proof}
We have 
$$ \min_{t\in[0,\ldots,T]}\mathbb{E}\|\nabla f(w^t)\|^2\leq \frac{1}{T+1} \sum_{t=0}^T\mathbb{E} \|\nabla f(w^t)\|^2.$$
The rest  is direct from the previous proposition. 
\end{proof}
From the above corollary we can observe that 
$\min_{t\in[0,\ldots,T]}\mathbb{E}\|\nabla f(w^t)\|^2$ converges to zero with the rate $1/\sqrt{T+1}$ which is the same as the classical rate known in the litterature for baseline SGD and its variants. 

We note that in the analysis, for simplicity, we use a fixed learning rate.  One can easily derive the convergence
of Algorithm \ref{alg:GBSFL} by choosing a sufficiently small or decreasing
learning rate, similar to the classical analysis of SGD.

In the next section, we show empirically the performance of our client selection strategy when a subset of devices is selected.

\section{Simulation Results}
\begin{figure}
\centering
\subfloat[Test Accuracy\label{testaccMN25}]{%
  \includegraphics[width=0.43\textwidth]{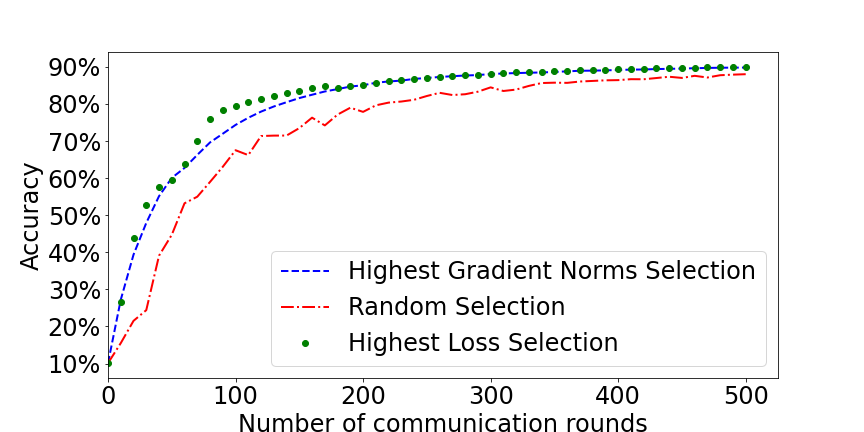}%
  }\par
\subfloat[Train Loss\label{trainlossMN25}]{%
  \includegraphics[width=0.43\textwidth]{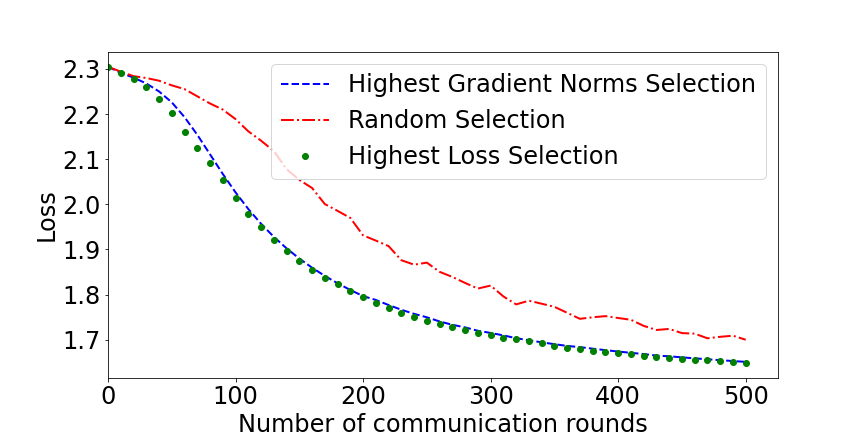}%
  }
\caption{Test accuracy and train loss on MNIST dataset, 25 selected devices, $\beta = 0.3$}
\label{MN25}

\end{figure}

\begin{figure}
\centering
\subfloat[Test Accuracy\label{testaccMN25}]{%
  \includegraphics[width=0.43\textwidth]{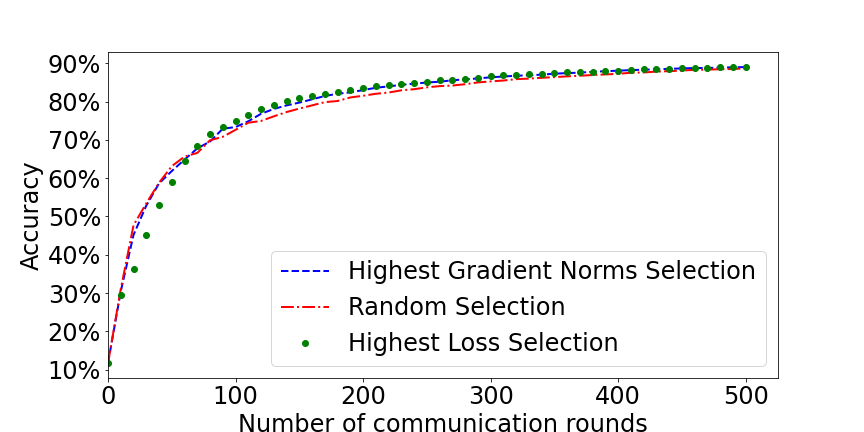}%
  }\par
\subfloat[Train Loss\label{trainlossMN25}]{%
  \includegraphics[width=0.43\textwidth]{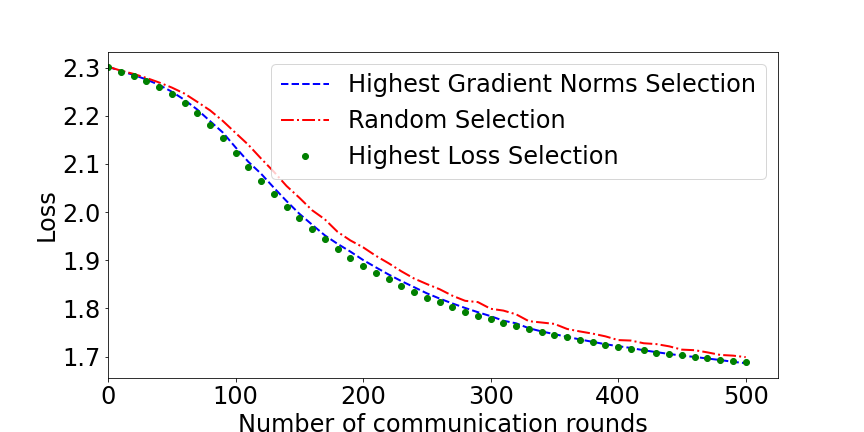}%
  }
\caption{Test accuracy and train loss on MNIST dataset, 25 selected devices, $\beta = 5$}
\label{B5MN25}

\end{figure}

\begin{figure}
\centering
\subfloat[Test Accuracy\label{testaccFm25}]{%
  \includegraphics[width=0.43\textwidth]{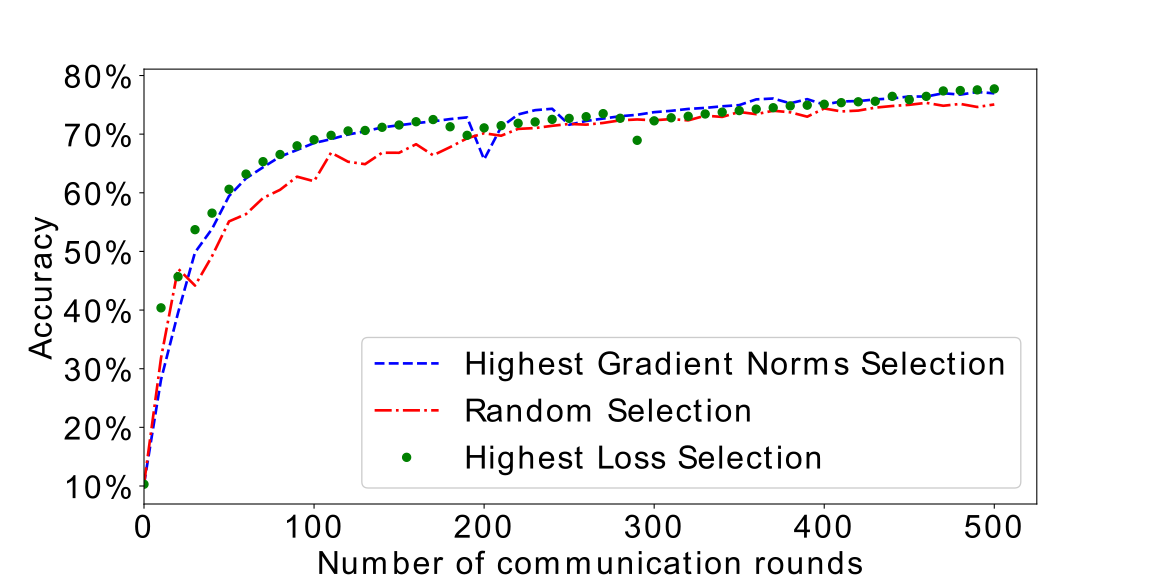}%
  }\par
\subfloat[Train Loss\label{trainlossFm25}]{%
  \includegraphics[width=0.43\textwidth]{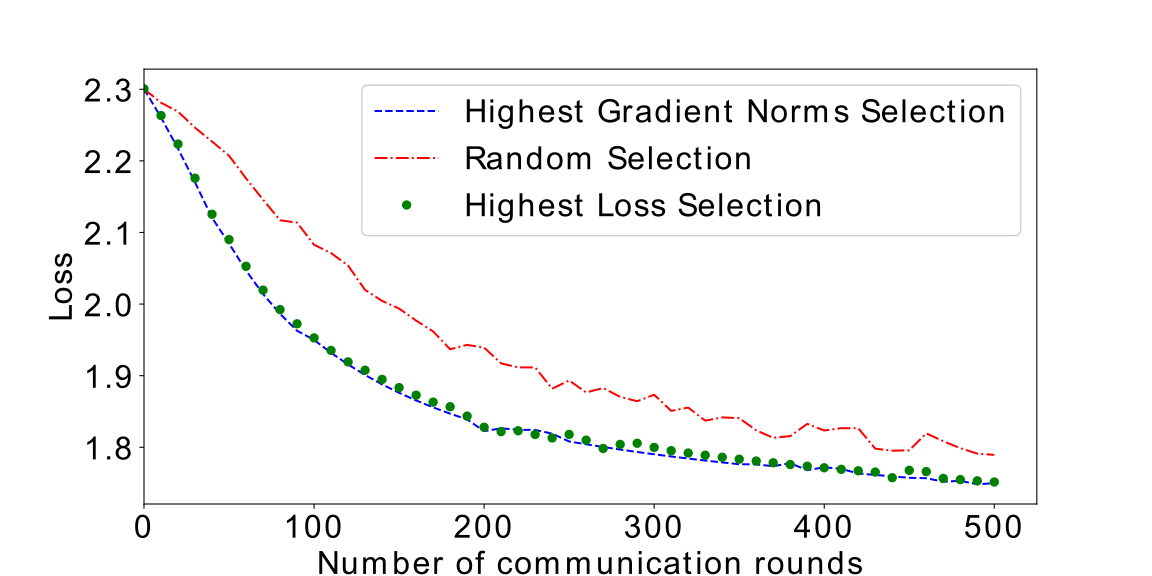}%
  }
\caption{Test accuracy and train loss on FMNIST dataset, 25 selected devices}
\label{FM25}
\end{figure}
In this section, we assess the performance of our proposed approach and compare it with random selection and highest loss selection approaches. We also give insights about the optimal number of devices to be selected following our strategy.
\subsection{Simulation Setup and performance evaluation}
Our experiments are conducted using Keras with Tensorflow. The plots in this paper are the results using SGD optimizer. We also performed experiments using ADAM optimizer and the results are similar to those obtained using SGD. We explore different datasets.
\begin{itemize}
  \item \textbf{MNIST}. We train a $3$ layers Multilayer Perceptron (MLP) with $199,210$ parameters on a non-iid partitioned MNIST~\cite{ref7}, a dataset of hand-written numbers from $0$ to $9$.
  \item \textbf{FMNIST}. We train a $3$ layers MLP with $199,210$ parameters on a non-iid partitioned FMNIST~\cite{ref8}, a dataset of fashion products from $10$ categories.
  \item \textbf{CIFAR-10}. We train a $3$ layers MLP with $656,810$ parameters on a non-iid partitioned CIFAR-10~\cite{CF10}, a dataset of colored images from $10$ categories.
\end{itemize}
We follow an approach similar to the one described in ~\cite{ref9} to partition the data in a non-iid manner between the devices. We use Dirichlet distribution $Dir_K(\beta)$ to allocate different amounts of data samples (quantity skew) and different labels (label distribution skew) across devices. $\beta$ is the concentration parameter ($\beta > 0$) used to control the degree of data imbalance level. A small $\beta$ implies large data heterogeneity. In our experiments, we select the learning rate 
by using a grid search. We select $25$ devices from $100$ and perform $500$ iterations. We compare our approach with the random selection and the highest loss selection. For the random approach, we perform $5$ runs and present the average metric. For MNIST dataset, we use two different values of $\beta$ to assess the impact of the data heterogeneity. When $\beta = 0.3$, i.e., large data heterogeneity, our approach gives quite similar results to the highest loss selection strategy and outperforms the random one as shown in Figure~\ref{MN25}. For example, at iteration $150$, we obtain an increase of accuracy of $14\%$ and a decrease of loss of $8\%$ using our strategy versus the random one.  However, when $\beta = 5$, i.e., less data heterogeneity, the random approach performs almost as good as its opponents, as shown in Figure~\ref{B5MN25}. This is due to the fact that when devices have similar data, their gradients are also similar. We confirm our findings by running experiments on FMNIST dataset as in Figure~\ref{FM25} and on CIFAR-10 as in Figure~\ref{CF25} although the performance is poor in the latter. Indeed, after hundreds of rounds, we could barely reach the $50\%$ accuracy. Due to space limitations, we only display results for small $\beta = 0.3$ for both FMNIST and CIFAR-10 datasets.
\begin{figure}
\centering
\subfloat[Test Accuracy\label{CFtestAcc25}]{%
  \includegraphics[width=0.43\textwidth]{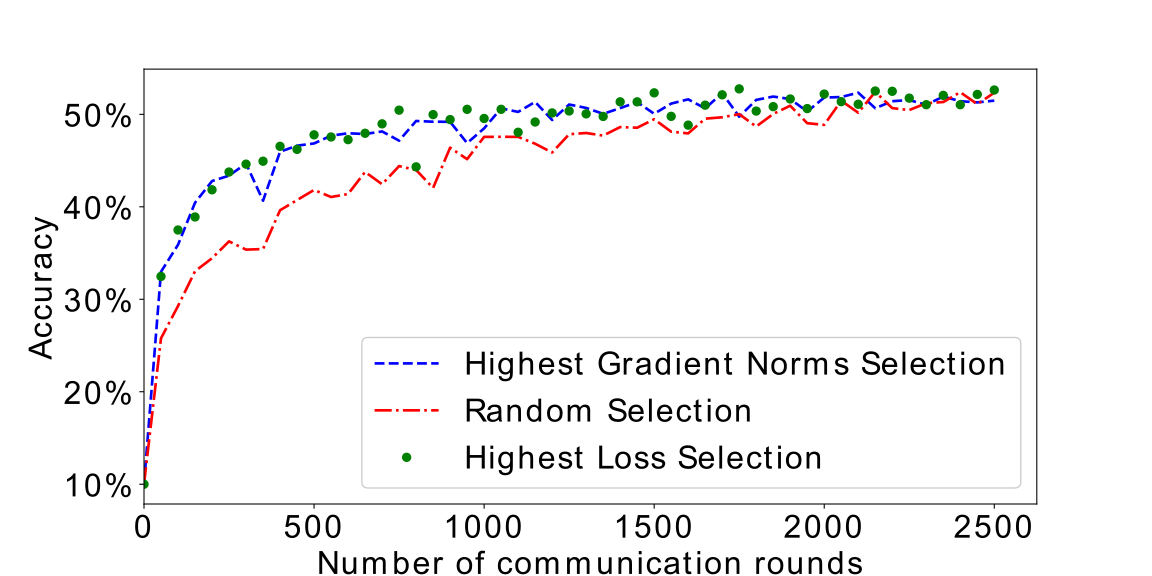}%
  }\par
\subfloat[Train Loss\label{CFTrainloss25}]{%
  \includegraphics[width=0.43\textwidth]{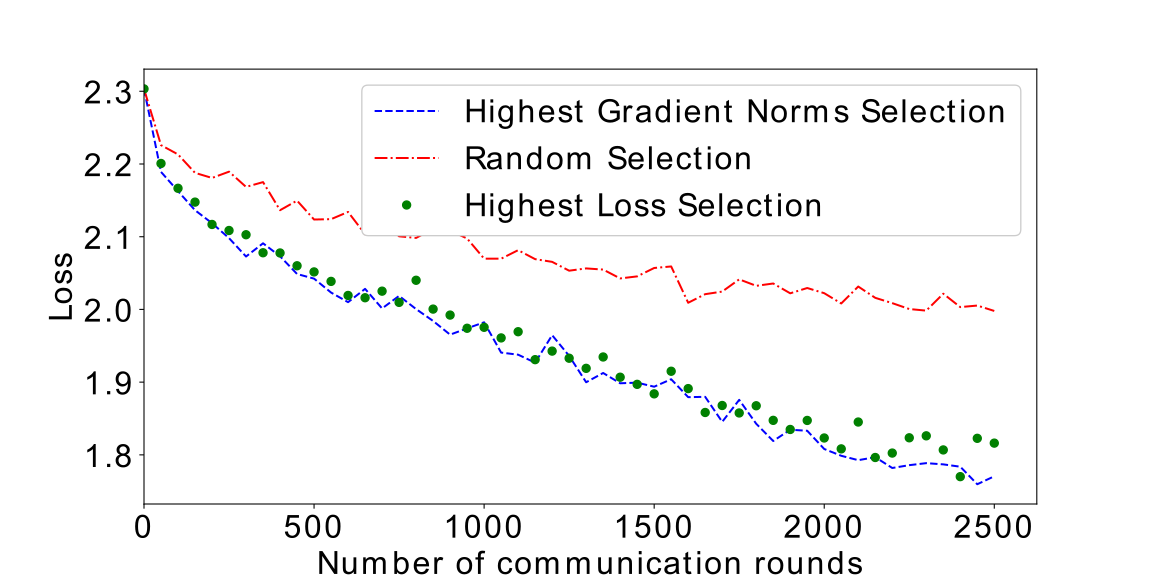}%
  }
\caption{Test accuracy and train loss on CIFAR-10 dataset, 25 selected devices}
\label{CF25}
\end{figure}
\subsection{Comparison between the highest norm selection and the highest loss selection}
For our proposed approach, the selection is based on the gradients calculated by the devices with the purpose to be shared with the central server. As a consequence, for some ML models where gradients can be estimated directly, there is no need to do any additional computations (except the computation of the gradients norms), unlike the highest loss approach where the losses should be computed at each communication round. Hence, with a limited number of computations, we can achieve similar performance as with the highest loss selection. Moreover, when the number of selected devices is high, the performance of the two approaches is almost the same as depicted in Figure \ref{FM85} where the curves of the two approaches are overlapping for $85$ selected devices among $100$ in the FMNIST dataset.
\begin{figure}
\centering
\subfloat[Test Accuracy\label{testAcc85}]{%
  \includegraphics[width=0.43\textwidth]{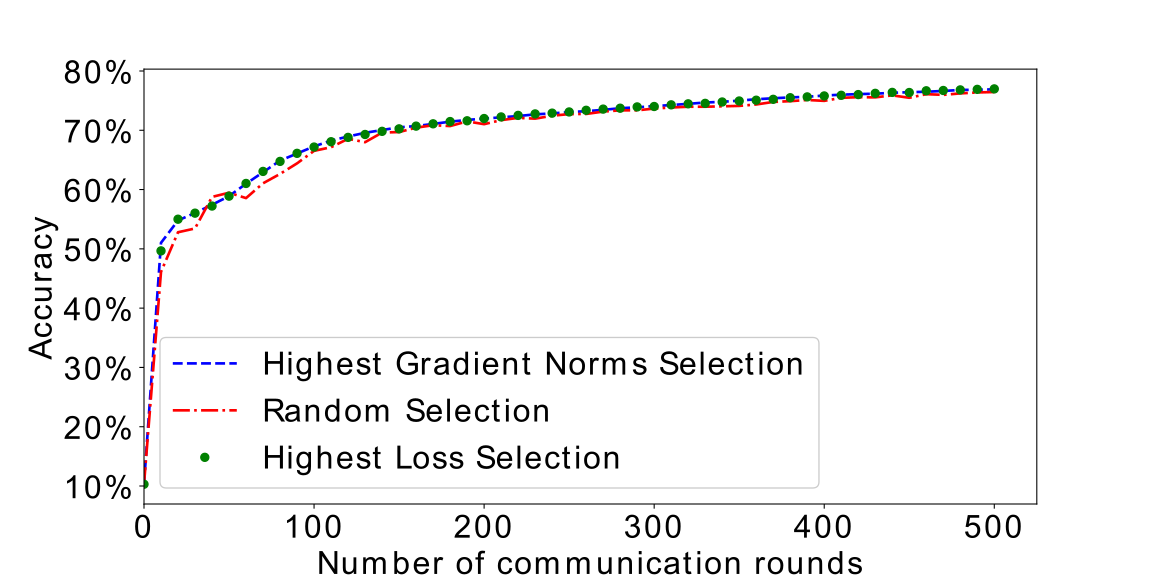}%
  }\par
\subfloat[Train Loss\label{Trainloss85}]{%
  \includegraphics[width=0.43\textwidth]{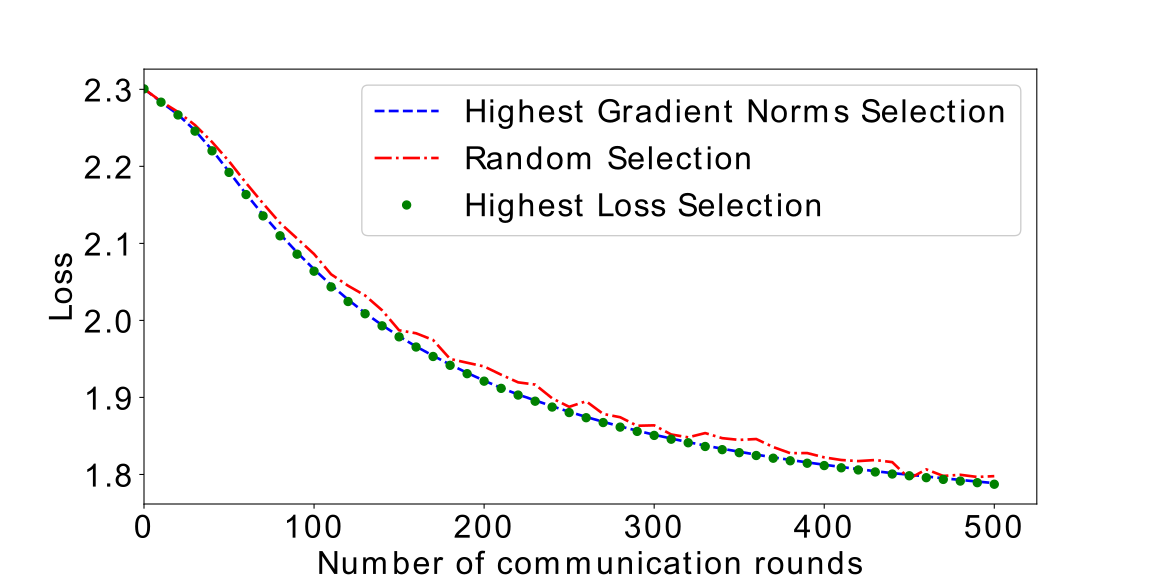}%
  }
\caption{Test accuracy and train loss on FMNIST dataset, 85 selected devices}
\label{FM85}
\end{figure}

\subsection{Optimal number of selected devices}
The performance of the selection approach is tightly related to the number of selected devices. Table \ref{tab150} and Table \ref{tab500} show the impact of the number of selected devices on the achieved test accuracy. For example, selecting one device while training MNIST dataset achieves $41\%$ accuracy at iteration $150$. $500$ rounds are needed in order to reach an accuracy of $84\%$. On the other side, by selecting $25$ devices, we can achieve the accuracy of $82\%$ after only $150$ iterations. In fact, when the number of selected devices is too small, the selected labels do not reflect the diversity of the entire data available at devices. This is more likely to happen in a highly skewed/ non-iid setup and can be the main cause of a ML model underfitting. Contrarily, selecting more devices does not necessarily lead to a  better generalization as it may cause an overfitting issue. The optimal number of selected devices depends on the neural network model and the dataset. In the studied scenarios, selecting $25$ devices provides the best results for both MNIST and CIFAR-10 datasets, while selecting $15$ devices achieves the best results for FMNIST dataset. It is worth mentioning that we also run our experiments for a higher number of selected devices ($95$) and full selection ($100$). The accuracy results are similar to those for $85$ selected devices.
\begin{table}[h]
\centering 
\begin{tabular}{c rrrrrrr} 
\hline\hline 
&\multicolumn{7}{c}{Number of selected devices} \\ [0.5ex]
\hline 
\textbf{Dataset} & 1 & 3 & 5& 15& 25& 50& 85\\ 
MNIST & 40.9 & 71.2 & 80.9& 79.4&  81.6& 80.7& 78.3\\
FMNIST & 52.1 & 62.8 & 62& 71.6& 71.5& 71.1& 70.5\\
CIFAR-10 & 10.0 & 18.2 & 32.9& 38.0& 40.4& 40.3& 40.3\\[1ex]
\hline 
\end{tabular}
\caption{Accuracy (\%) per number of selected devices at communication round 150, following the Highest Gradient Norms Selection} 
\label{tab150}
\end{table}

\begin{table}[h]
\centering 
\begin{tabular}{c rrrrrrr} 
\hline\hline 
&\multicolumn{7}{c}{Number of selected devices} \\ [0.5ex]
\hline 
\textbf{Dataset} & 1 & 3 & 5& 15& 25& 50& 85\\ 
MNIST & 83.6 & 88.6 & 90.0& 89.9&  89.9& 89.4& 88.8\\
FMNIST & 70.9 & 74.9 & 77.7& 78.1& 77.4& 77.8& 77.5\\
CIFAR-10 & 10.0 & 28.6 & 38.4& 46.7& 47.6& 47.2& 47.5\\[1ex]
\hline 
\end{tabular}
\caption{Accuracy (\%) per number of selected devices at communication round 500, following the Highest Gradient Norms Selection} 
\label{tab500}
\end{table}
\section{Conclusion}
 In this paper, we have presented an efficient method for device selection in FL by using the norms of gradients. We have provided theoretical convergence guarantees for our algorithm. Our experiments performed on multiple datasets confirm the efficiency of the proposed approach. 
In ongoing work, we will explore the combination of our selection method with gradient compression techniques e.g., Top-$k$ to further reduce communication costs. 
\bibliographystyle{IEEEbib}
\bibliography{references}


\end{document}